\documentclass[sigconf,nonacm]{acmart}
\AtBeginDocument{%
  }

\setcopyright{acmlicensed}
\copyrightyear{2025}
\acmYear{2025}
\acmDOI{}
\acmConference[CIKM '25]{Conference on Information and Knowledge Management}{November 10--14, 2025}{Coex, Seoul, Korea}

\usepackage{amsmath}
\usepackage{mathtools}
\usepackage{amsthm}
\usepackage{amsfonts}

\theoremstyle{plain}
\newtheorem{theorem}{Theorem}[section]
\newtheorem{proposition}[theorem]{Proposition}
\newtheorem{lemma}[theorem]{Lemma}

\theoremstyle{definition}

\newtheorem{assumption}[theorem]{Assumption}
\theoremstyle{remark}

\usepackage{algorithm}
\usepackage{algorithmic}
\usepackage{booktabs}
\usepackage{xcolor} 

\usepackage{enumitem}

\newcommand{\E}{\mathbb{E}} 
\renewcommand{\mid}{\,|\,}
\newcommand{\yres}[1]{R_{#1}}
\newcommand{\dres}[1]{V_{#1}}
\newcommand{\yltFun}{f_{\hspace{-0.2mm}u\hspace{-0.2mm}}}
\newcommand{\dltFun}{f_{\hspace{-0.2mm}v\hspace{-0.2mm}}}
\newcommand{\yltNoise}{W_u}
\newcommand{\dltNoise}{W_v}

\begin{document}

\title{Latent Variable Modeling for Robust Causal Effect Estimation\\[10mm]}


\author{Tetsuro Morimura}
\affiliation{%
  \institution{CyberAgent}
  \city{Tokyo}
  \country{Japan\\[10mm]}
}

\author{Tatsushi Oka}
\affiliation{%
  \institution{Keio University}
  \city{Tokyo}
  \country{Japan}
}

\author{Yugo Suzuki}
\authornote{This work was done during an internship at CyberAgent.}
\affiliation{%
 \institution{Yokohama City University}
 \city{Kanagawa}
 \country{Japan\\[10mm]}
}

\author{Daisuke Moriwaki}
\affiliation{%
  \institution{CyberAgent}
  \city{Tokyo}
  \country{Japan}}

\renewcommand{\shortauthors}{}

\begin{abstract}
Latent variable models provide a powerful framework for incorporating and inferring unobserved factors in observational data. In causal inference, they help account for hidden factors influencing treatment or outcome, thereby addressing challenges posed by missing or unmeasured covariates. This paper proposes a new framework that integrates latent variable modeling into the double machine learning (DML) paradigm to enable robust causal effect estimation in the presence of such hidden factors. We consider two scenarios: one where a latent variable affects only the outcome, and another where it may influence both treatment and outcome. To ensure tractability, we incorporate latent variables only in the second stage of DML, separating representation learning from latent inference. We demonstrate the robustness and effectiveness of our method through extensive experiments on both synthetic and real-world datasets.
\end{abstract}




\maketitle

\section{Introduction}
\label{sec:intro}
Central to data science are both predictive analytics and causal inference. 
The primary goal of the latter is to estimate the effect of one variable, known as a {\it treatment}, on another, known as an {\it outcome}, using observational data \citep{Holland_causal_1986,Pearl_causal_2009,Athey_essay_2017}. 
This task is essential in various domains,
where it is required to understand how a treatment impacts an outcome of interest amidst numerous covariates \citep{Imbens_econometrics_2009, causal_inference_survey_2021}. 
For instance, understanding the efficacy of medical treatments is paramount in healthcare. 
Researchers often seek to estimate the effect of drug dosage on patient outcomes,
such as recovery rates or symptom alleviation,
while accounting for patient-specific characteristics and clinical histories \citep{causality_medical_imaging_2020,causality_patient_trajectory_2023}. 
Similarly, in marketing, determining the impact of advertising campaigns on consumer behavior is crucial.
Analysts might analyze how different levels of ad exposure influence consumer purchases,
taking into account various factors such as demographic data and previous buying behavior \citep{causal_ad_time_series_2015,uplift_survey_2021}.

In practical causal inference settings, ensuring the completeness of covariates is often challenging, as highlighted in several studies \citep{ignorability_example_2016,ignorability_example_2017,dml_ignorability_example_2021}.
To address this, we introduce latent variables to account for unobserved factors that may influence the outcome or treatment but are not captured by observed covariates.
This approach allows us to incorporate domain knowledge about the data-generating process, such as structured noise, into the model.
Even when such knowledge is limited, we propose model selection using information criteria to choose appropriate latent structures.

While latent variable modeling offers a powerful tool for incorporating unobserved factors into causal inference, it often involves iterative procedures, such as alternating between model parameter updates and latent variable inference, that can be computationally intensive. 
This computational challenge is particularly pronounced when combined with large-scale machine learning models, which are increasingly used in causal inference and already require substantial optimization efforts.
To mitigate this, we adopt the double machine learning (DML) framework~\citep{dml_2018}, which decouples the estimation of nuisance components (e.g., outcome and treatment models) from the final causal estimation.
Our key idea is to introduce latent variables only into the second stage of DML, enabling efficient estimation of causal effects, while keeping the overall computation tractable.
We refer to this framework as \textit{latent double machine learning} (latent DML). 

\begin{figure*}[t]
 \centering
 \includegraphics[width=1.\linewidth]{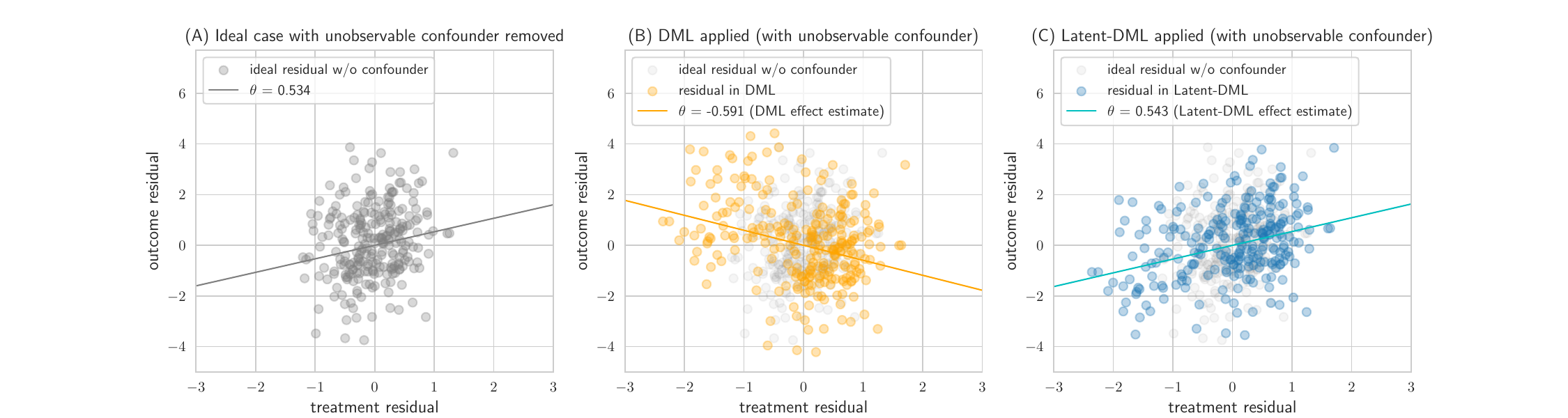}
 \Description{Comparison of residual plots: (A) ideal case without confounder, (B) DML with confounder (biased), (C) Latent DML correcting bias.}
 \caption{Causal effect estimation on treatment and outcome residuals in DML:
 Residuals are derived from the first stage of DML, where ML models estimate outcomes and treatments given observed covariates.
 The plots illustrate the second stage (regression analysis), where the relationship between the residuals informs causal effect estimation, with the true effect set to $\theta = 0.5$. 
 (A) The ideal case, with unobserved confounders removed, results in accurate estimation ($\hat\theta = 0.53$).
 (B) Applying standard DML in the presence of an unobserved confounder yields a biased estimate ($\hat\theta = -0.59$).
 (C) Latent DML adjusts the outcome residuals using estimated latent variables, producing an improved estimate ($\hat\theta = 0.54$) close to the true effect.}
 \label{fig:teaser}
\end{figure*}

We consider two distinct modeling scenarios involving latent variables.
The first scenario assumes that the latent variable affects only the outcome.
For instance, in advertising effectiveness estimation, unobserved social media buzz can significantly amplify the observable impact, such as Return on Advertising Spend (ROAS), without directly influencing the advertising strategies (treatment). 
The second scenario assumes that the latent variable influences both the treatment and the outcome, acting as an unobserved confounder.
One example is a competitor’s unrecorded marketing campaign, which may affect both a store’s marketing decisions (treatment) and its sales figures (outcome).
Another example arises in program evaluation studies, where unobserved participant characteristics (e.g., parental income) may influence both the duration of program participation (treatment) and long-term outcomes, potentially biasing the causal effect estimation.
Figure~\ref{fig:teaser} illustrates the second scenario by comparing standard DML and our proposed latent DML in the presence of unobserved confounding (see Section~\ref{sec:experiment_synthetic} for details).

%
%
Contributions of this paper include:
\begin{itemize}[leftmargin=*, itemsep=1pt, parsep=0pt]
\item{\bf Integration of latent variable modeling into DML:} 
We propose \textit{latent double machine learning} (latent DML), a novel framework that integrates latent variable modeling into the DML paradigm to address unobserved factors in causal inference.	
To the best of our knowledge, this is the first work to develop a general DML-based method tailored to latent variable models.
\item{\bf Algorithmic implementation:}
We present two implementations of latent DML, including adaptations of the EM algorithm. 
\item{\bf Comprehensive validation:}
We demonstrate the effectiveness of latent DML through extensive empirical validation and provide theoretical guarantees. 
\end{itemize}
\section{Preliminary}
\label{sec:preliminary}
This section provides an overview of the key concepts in DML and latent variable modeling, which are essential for our paper.

\subsection{Double machine learning}
\label{sec:dml}
Double machine learning (DML) is a framework that integrates machine learning techniques with causal inference, 
primarily focusing on estimating treatment effects in the presence of high-dimensional covariates \citep{dml_2018}.
This approach builds upon the partially linear model proposed by \cite{Robinson_1988}, 
enhanced by the incorporation of cross-fitting and the concept of Neyman orthogonality.

A key characteristic of Robinson's partially linear model is its ability to separate the estimation of causal effects from the nuisance components, which are parts of the model not directly related to the causal effect but potentially correlated with the treatment and outcome. 
The model is formulated as follows for $N$ i.i.d. data points $\{(X_i, D_i, Y_i)\}_{i=1}^{N}$:
\begin{align}
 Y_i &= \theta D_i + g(X_i) + U_i, 
 \label{eq:dml-outcome}
 \\
 D_i &= m(X_i) + V_i,
 \label{eq:dml-treatment}
\end{align}
where $X_i \in \mathbb{R}^d$ is the covariates, 
$D_i \in \mathbb{R}$ is the treatment variable, 
$Y_i \in \mathbb{R}$ is the outcome, 
and $\theta \in \mathbb{R}$ is the average causal effect. 
The unknown functions $g:\mathbb{R}^d\rightarrow\mathbb{R}$ and $m:\mathbb{R}^d\rightarrow\mathbb{R}$ represent the influence of covariates on the outcome and treatment, respectively.
The inferential target is the causal effect parameter $\theta$, 
while $g$ and $m$ are {\it nuisance} components, in the sense that their estimation is not of direct interest.
The noise terms $U_i$ and $V_i$ are assumed to be zero mean and independent of the covariates, that is,
\begin{align}
 &\mathbb{E}[U_i \mid X_i] = \mathbb{E}[V_i \mid X_i] = 0
 \label{eq:noise_zero_condition}
 \\
 &\mathbb{E}[U_i V_i \mid X_i] = 0.
 \label{eq:unconfoundedness}
\end{align}
Equation~\eqref{eq:unconfoundedness} ensures there is no unobservable confounder and is known as the unconfoundedness condition.

In the DML framework following Robinson's approach, 
the first step is to estimate the functions $m(X)=\mathbb{E}[D\,|\,X]$ and
$$h(X) := \mathbb{E}[Y \,|\, X ] = g(X) + \theta m(X).$$
This is achieved by employing machine learning models to approximate these functions.
Specifically, a model is used to approximate the expected treatment as $\hat{m}(X)$ and another model to approximate the expected outcome as $\hat{h}(X)$.
Once these models are fitted, the estimated functions $\hat{m}(X)$ and $\hat{h}(X)$ are used to compute the residuals:
\begin{align}
\hat{\yres{}} = Y - \hat{h}(X),
\quad 
\hat{\dres{}} = D - \hat{m}(X).
\label{eq:calc_residuals}
\end{align}
It is important to note that $\hat{\yres{}}$ represents the residual of the outcome, 
differing from the original noise term $U_i$ in the model. 
Specifically, the outcome residual is computed as
\begin{equation}
\yres{} := Y - h(X) =\theta V + U,
\label{eq:y_residual}
\end{equation}
reflecting the influence of both the causal effect and the original noise.
In contrast, the residual of the treatment corresponds directly to the original noise term $V$.
Finally, the causal effect estimate in DML is obtained using the formula:
\begin{align}
 \textstyle
 \hat{\theta} := \left(\sum_{i=1}^N \hat{\dres{i}}^2\right)^{\!-1} \sum_{i=1}^N \hat{\dres{i}} \hat{\yres{i}}
 \label{eq:calc_causal_effect}
\end{align}
This approach effectively isolates the causal effect $\theta$ from the nuisance components represented by $g(X)$ and $m(X)$.

Other feature of DML is the implementation of the cross-fitting procedure. 
In this process, the data is divided into several folds.
Each fold is used to calculate residuals $\hat{U}, \hat{V}$ for the estimation of causal effects, 
while the remaining data serve as the training set for the machine learning models $\hat{h}, \hat{m}$. 
For a detailed description of this procedure, see Algorithm \ref{algo:dml}.
This technique is pivotal in reducing overfitting, particularly in high-dimensional settings, and enhances the convergence rate of the estimators, thereby improving the overall robustness of the causal effect estimation.
%


\begin{algorithm}[H]
\caption{DML}
\label{algo:dml}
\begin{algorithmic}[1] 
\STATE \textbf{Input:} Data $\{(X_i, D_i, Y_i)\}_{i=1}^n$.
\STATE \textit{Step 1: Estimate residuals}
\STATE Split the data into $K$ folds: $\{\mathcal{T}_1,\dots,\mathcal{T}_K\}$.
\FOR{each fold $k = 1, \ldots, K$}
    \STATE Define $\mathcal{T}_{-k}$ as all data excluding fold $k$.
    \STATE Train ML models $\hat{h}(X)$ and $\hat{m}(X)$ on $\mathcal{T}_{-k}$.
    \STATE Compute residuals $\hat{\yres{}}$ and $\hat{V}$ for $\mathcal{T}_k$ using Eq.~\eqref{eq:calc_residuals}.
\ENDFOR
\STATE \textit{Step 2: Estimate causal effect}
\STATE Aggregate the residuals from all folds.
\STATE Calculate the causal effect estimate $\hat\theta$ using Eq.~\eqref{eq:calc_causal_effect}.\hspace{-2mm}
\STATE \textbf{Output:} Estimated causal effect $\hat{\theta}$.
\end{algorithmic}
\end{algorithm}

\subsection{Latent variable modeling}
Latent variable modeling is a statistical technique that introduces inductive biases into probabilistic models by incorporating hidden variables, providing additional structure and enhancing model estimation.
This approach is often applied in scenarios where direct measurement of certain variables is not feasible, as it helps to account for hidden factors influencing the observed data.
Common methods for latent variable models include the Expectation-Maximization (EM) algorithm, variational Bayes, and Markov chain Monte Carlo methods \citep{Bishop06_PRML}.

We will use the EM algorithm in this study for its effectiveness and simplicity. 
Here, we assume a parametric stochastic model 
$p(\{( X_i, Z_i)\}_{i=1}^N \mid \gamma)=\prod_{i=1}^N p(X_i, Z_i \mid \gamma)$,
where $X_i$, $Z_i$, and $\gamma$ are 
observed data, latent variable, and model parameter, respectively.
For simplicity, 
we may notate $p(\{(X_i, Z_i)\}_{i=1}^N \mid \gamma)$ as 
$p(X, Z \mid \gamma)$.
The EM algorithm iteratively performs two steps until convergence to locally maximize the log likelihood function $\log p(X \mid \gamma)$:
\begin{itemize}[leftmargin=*, itemsep=1pt, parsep=0pt]
 \item {\bf Expectation step (E-step)}: 
       Compute the following expectation with respect to the conditional distribution of the latent variable $Z$ given the data $X$ and the current parameter estimate $\gamma^{(t)}$:
       \begin{equation}
	    Q(\gamma \mid \gamma^{(t)}) 
	  = 
	  \mathbb{E}_{Z\sim p(Z \mid X,\gamma^{(t)})}
	  \big[\log p(X, Z \mid \gamma) \big].
	  \notag
       \end{equation}
  \item {\bf Maximization step (M-step)}: 
	Update the parameter estimate by maximizing $Q(\gamma \mid \gamma^{(t)})$: 
	\begin{equation}
	 \gamma^{(t+1)} = \arg \max_\gamma Q(\gamma \mid \gamma^{(t)}).
	  \label{eq:em_m}
	\end{equation}
 \end{itemize}

\section{Latent double machine learning}
\label{sec:latent-dml}
%

We introduce \textit{latent double machine learning} (latent DML), a framework that integrates latent variable modeling into the double machine learning (DML) setting for improved causal effect estimation. 
Section~\ref{sec:latent-dml-approach} presents the overall design, while Sections~\ref{sec:latent-dml-outcome-only} and \ref{sec:latent-dml-confounder} describe two concrete instantiations of the framework. Section~\ref{sec:latent_dml_practical} discusses practical implementation considerations.

\subsection{Latent DML framework} 
\label{sec:latent-dml-approach}

We aim to address challenges in causal inference arising from unobserved factors that are not captured by observed covariates but may influence the treatment or outcome.
To this end, we propose a latent DML framework that incorporates latent variable modeling into the DML process, combining the flexibility of latent variable models with the computational and statistical advantages of DML.

Since DML inherently follows a two-step procedure, there are multiple ways to incorporate latent variables. 
However, due to the iterative nature of latent variable modeling, incorporating it into the first step would require repeated training of machine learning models, which is computationally demanding.
Therefore, we introduce a latent variable only into the second step of DML, specifically in lines $10$ and $11$ of Algorithm \ref{algo:dml}. This design allows for efficient latent variable inference without retraining the nuisance components.

%

%
%
%
%

In the latent DML framework, we need to specify the distribution of the outcome residual $R=\theta V + U$ in Eq.~\eqref{eq:y_residual}, which consists of the noise terms $U$ and $V$.
%
We consider the following generalized structure for the noise models
with parameterized functions $\yltFun$ and $\dltFun$ of a latent variable $Z$:
\begin{align}
 U = \yltFun(Z) + \yltNoise, \quad 
 V = \dltFun(Z) + \dltNoise.
 \label{eq:noise-models}
\end{align}
Here, $Z$, $\yltNoise$, and $\dltNoise$ are random variables, assumed to be mutually independent
and also independent of the covariate $X$.
The modeling of these functions and variables, while not trivial, allows for the incorporation of domain knowledge into the model, 
which could be advantageous when such prior information is available \cite{bayesian_data_analysis_book_2013,bayes_causal_inference_2022}. 
Specific instances are discussed in the subsequent subsections,
providing two distinct scenarios within our latent DML framework,
an {\it outcome-only latent variable} scenario in Section~\ref{sec:latent-dml-outcome-only}, 
and an {\it unobservable confounder} scenario in Section~\ref{sec:latent-dml-confounder}.

Most importantly, this framework allows $\mathbb{E}[\yltFun(Z)\dltFun(Z) \mid X]$ to be nonzero. 
Thus, it does not require the unconfoundedness condition of Eq.~\eqref{eq:unconfoundedness} and can handle unobservable confounders when the noise models are well specified.

We denote the set of all parameters involved in the residual modeling as $\gamma$. 
The set $\gamma$ includes parameters characterizing the functions $\yltFun$ and $\dltFun$, parameters for the random variables $Z$, $\yltNoise$, and $\dltNoise$, and a causal effect parameter $\hat\theta$.
The objective function for this modeling is the negative log-likelihood,
\begin{align}
 \textstyle
 L (\gamma) & = - \log p(R, V \mid \gamma ) \notag \\
 & = - \log \int_Z p(R, V\mid Z, \gamma) p(Z\mid\gamma) dZ.
 \label{eq:objective_lvm}
\end{align}
We estimate the parameter $\gamma$ by minimizing $L(\gamma)$ using the residual samples $\{(\hat R_i, \hat V_i)\}_{i=1}^N$ with a latent variable modeling method such as the EM algorithm, and denote the obtained estimate as $\hat\gamma$.
If the model space contains the true data-generating process with $\gamma_0$, then $\hat\gamma$ is expected to be a consistent estimator of $\gamma_0$.
A more rigorous theoretical analysis of this result is provided in Appendix \ref{sec:latent_dml_theory}.

Latent DML estimates the causal effect $\hat\theta$ by solving the following score equation, which is derived from moment conditions:
%
\begin{align}
 \textstyle
 \sum_{i=1}^N(\hat R^{\rm z}_i - \hat \theta \hat V_i) \hat V_i = 0,
 \label{eq:score-function}
\end{align}
where $\hat R^{\rm z}_i$ is an adjusted outcome residual defined as
\begin{align}
 \hat R^{\rm z}_i &:= \hat{\yres{i}} - \mathbb{E}[{\yltFun}(Z) \mid \hat{\yres{i}}, \hat{V}_i, \hat\gamma].
 \label{eq:adjusted-outcome-residual}
\end{align}
%
$\hat\theta$ is an unbiased estimate of the causal effect if the true model is included in the model space and the the nuisance components are estimated as $m=\hat m$ and $h=\hat h$.
The unbiasedness follows from
from Eqs.~\eqref{eq:y_residual}, \eqref{eq:noise-models}, and \eqref{eq:adjusted-outcome-residual}, where the error term $\hat R^{\rm z} - \hat\theta \hat V$ under the true parameter $\gamma_0$ satisfies:
$$R^z - \theta V = \yltFun(Z)  - \mathbb{E}[\yltFun(Z) \mid R, V, \gamma_0] + \yltNoise,$$
and thus the following condition for the unbiasedness holds:
\begin{align*}
    \mathbb{E}[R^z - \theta V \,|\, V] 
    &=
    \mathbb{E}[ \yltFun(Z)  - \mathbb{E}\left[\yltFun(Z) \mid R, V, \gamma_0] + \yltNoise \,|\, V\right]
    \\&
    =0, 
\end{align*}
since \(Z\), \(\yltNoise\), and \(\dltNoise\) are assumed to be mutually independent.
We summarize the procedure in Algorithm~\ref{algo:latent-dml}.
\begin{center}
\begin{minipage}{0.467\textwidth}
\begin{algorithm}[H]
\caption{Latent DML {\color{blue}(Additions to DML shown in blue)}}
\label{algo:latent-dml}
\begin{algorithmic}[1] 
\STATE \textbf{Input:} Data $\{(X_i, D_i, Y_i)\}_{i=1}^n$.
\STATE \textit{Step 1: Estimate residuals} 
\STATE Split the data into $K$ folds: $\{\mathcal{T}_1,\dots,\mathcal{T}_K\}$.
\FOR{each fold $k = 1, \ldots, K$}
    \STATE Define $\mathcal{T}_{-k}$ as all data excluding fold $k$.
    \STATE Train ML models $\hat{h}(X)$ and $\hat{m}(X)$ on $\mathcal{T}_{-k}$.
    \STATE Compute residuals $\hat{\yres{}}$ and $\hat{V}$ for $\mathcal{T}_k$ using Eq.~\eqref{eq:calc_residuals}.
\ENDFOR
\STATE \textit{Step 2: Estimate causal effect}
\STATE Aggregate the residuals from all folds.
\STATE {\color{blue} Estimate the latent model parameter $\hat\gamma$ by minimizing the empirical mean of $L(\gamma)$ of Eq.~\eqref{eq:objective_lvm} over the residuals $\hat R$ and $\hat V$.}
\STATE {\color{blue} Adjust the outcome residual $\hat R$ according to Eq.~\eqref{eq:adjusted-outcome-residual}.}
\STATE Calculate the causal effect estimate $\hat\theta$ using Eq.~\eqref{eq:calc_causal_effect}.
\STATE \textbf{Output:} Estimated causal effect $\hat{\theta}$.
\end{algorithmic}
\end{algorithm}
\end{minipage}
\end{center}

\subsection{Outcome latent model}
\label{sec:latent-dml-outcome-only}
We present an outcome latent model with $Z \in \mathbb{R}$ for a {\it outcome-only latent variable} scenario,
where the noise models in Eq.~\eqref{eq:noise-models} use $\yltFun(Z)=Z$ and $\dltFun(Z)=0$,
meaning that the treatment residual $V$ is independent of $Z$ and does not need to be modeled explicitly.
We model the noise term $U$ as
\begin{align}
 U &= Z + \yltNoise, 
 \label{eq:u_expon}
\end{align}
where $\yltNoise$ follows a normal distribution $\mathcal{N}(0,\sigma^2)$
of zero mean and standard deviation $\sigma$. 
Owing to the zero-mean property of $U$ in Eq.\eqref{eq:dml-outcome}, the mean of $Z$ must also be zero.

Assuming a normal distribution for $Z$ causes an identifiability issue because it violates the non-singular condition of the Fisher information matrix in Proposition \ref{prop:Identifiability}.
However, other choices are permissible provided that they maintain a zero mean.
As one of the simplest implementations, we model $Z$ as a shifted exponential distribution.
Specifically, we assume:
\begin{align}
 Z \sim \mathcal{E}(\beta) - \beta,
 \label{eq:z_expon}
\end{align}
where $\mathcal{E}(\beta)$ denotes the exponential distribution with mean parameter $\beta$.
%
%

The use of an exponential distribution for modeling $Z$ is particularly pertinent in domains
where outcomes $Y$ are subject to sudden, substantial increases due to factors not captured by covariates $X$. 
For instance, in advertising effectiveness analysis, 
a viral buzz on social media can significantly increase online sales or website traffic.
In such cases, the exponential distribution effectively models the potential for upward shifts in these outcomes, 
providing a realistic representation of these unpredictable events.

The joint distribution is, 
since the residual of the outcome is $\yres{} = \theta V + U$,
\begin{align}
 &p(\yres{},Z \mid V, \theta,\beta, \sigma) 
 :=
 \mathcal{E}(Z+\beta;  \beta)\,
 \mathcal{N}(\yres{}; \theta V + Z-\beta, \sigma^2)
 \label{eq:outcome_latent_joint_dist}
\end{align}

Having defined the stochastic model with latent variables, 
we now turn our attention to parameter estimation. 
The optimal parameters are those that maximize the marginal log-likelihood 
$$p(\yres{} \,|\, V, \theta, \beta, \sigma) = \int_Z p(\yres{}, Z \,|\, V, \theta, \beta, \sigma) dZ,$$
and we employ the EM algorithm.
Each iteration of the EM algorithm involves updating the parameters in the M-step. 
While most parameters can be updated in closed form during this step, 
the optimization of $\beta$ in Eq.~\ref{eq:em_m} poses a unique challenge.
This optimization problem cannot be analytically solved,
so we employ the natural gradient method for updating $\beta$,
using the Fisher information of $\mathcal{E}(\beta)$ as a metric.

Upon convergence of the EM algorithm, the causal effect estimate $\hat\theta$ is calculated by solving the score equation~\eqref{eq:score-function} with
the adjusted outcome residual 
$\hat R^{\rm z}_i = \hat{\yres{i}} - \mathbb{E}[Z_i\,|\, \hat{\yres{i}}, \hat{V}_i, \theta, \beta, \sigma]$ of Eq.~\eqref{eq:adjusted-outcome-residual}.
According to the joint distribution of Eq.~\eqref{eq:outcome_latent_joint_dist}, 
$\mathbb{E}[Z_i \,|\, \hat{\yres{i}}, \hat{V}_i, \theta, \beta, \sigma]$
is calculated as the mean of a truncated distribution of $\mathcal{N}(m_i, \sigma^2)$, where  
$$m_i := \hat{\yres{i}} - \theta\hat{V}_i + \beta - \sigma^{2}/ \beta,$$ 
truncated over the interval $[0, \infty]$. 
%


\subsection{Confounder latent model}
\label{sec:latent-dml-confounder}
Unlike the previous subsection, 
here a latent variable $Z \in \mathbb{R}$ acts as an unobservable confounder, 
affecting both the treatment and the outcome. 
%
%
We set $\yltFun(Z) = aZ$ and $\dltFun(Z) = bZ$ in Eq.~\eqref{eq:noise-models} as 
\begin{align}
 U = a Z + \yltNoise, \quad 
 V = b Z + \dltNoise,
 \label{eq:confounder_latent_noise}
\end{align}
where $a\in\mathbb{R}$ and $b\in\mathbb{R}$ are unknown and represent the impact of $Z$.
$\yltNoise$ and $\dltNoise$ follow normal distributions $\mathcal{N}(0, \sigma_u)$ and $\mathcal{N}(0, \sigma_v)$, respectively. 
We model $Z$ as a shifted Bernoulli distribution:
\begin{align}
 Z \sim \mathcal{B}(q) - q,
 \notag
\end{align} 
where $\mathcal{B}(q)$ denotes the Bernoulli distribution with success probability $q\in(0,1)$, 
ensuring zero mean for $U$ and $V$.
The choice of this model for $Z$ is particularly suitable for scenarios involving unobservable, randomly occurring factors, independent of the observed covariates $X$. 
%
A relevant example of $Z$ could be a competitor's unrecorded marketing campaign, influencing both treatment and outcome while remaining unobserved.

Based on this setting, the joint distribution is defined as:
\begin{align}
 &p(\yres{},V, Z \,|\, \theta, a, b, q, \sigma_u, \sigma_v)
 \notag
 \\&\hspace{10mm}
 :=
 \mathcal{B}(Z;q)\,
 \mathcal{N}(V; b Z, \sigma^2_v)\,
 \mathcal{N}(\yres{}; \theta V + a Z,  \sigma^2_u).
 \notag
\end{align}

Similar to the previous section, we employ the EM algorithm for the parameter optimization.
%
In the M-step of this case, while most parameters can be updated in closed form,
we update $q$ using the natural gradient method, due to the challenges in obtaining an analytical solution. 
%

Once the EM algorithm have converged, 
the causal effect estimate $\hat\theta$ is calculated from the score equation~\eqref{eq:score-function} with
the following adjusted outcome residual of Eq.~\eqref{eq:adjusted-outcome-residual},
\begin{align}
 \hat R^{\rm z}_i :=  \hat{\yres{i}} - a(\pi_{i}-q),
 \notag
\end{align}
where $\pi_i$ is the following conditional probability of $Z$,
\begin{align*}
 \pi_i 
 & := p(Z_i=1-q \mid \hat{\yres{i}}, \hat{V}_i, \theta, a, b, q, \sigma_u, \sigma_v)
 \\&\hspace{1mm}
 =
 \frac{1}{1 \!+\! 
 \exp\!\!
 \bigg(\!\!
 - \frac{ (2q-1)a^{\!2} + 2a(\hat{\yres{i}}-\theta\hat{V}_i) }{2\sigma_u^{2}}
 - \frac{ (2q-1)b^{\!2} + 2b \hat{V}_i}{2\sigma_v^{2}}
 \!\bigg)
 }.
\end{align*}
This adjustment of the outcome residual $\hat{\yres{i}}$ to $\hat R^{\rm z}_i$ 
can be regarded as removing the influence of the confounder $Z$ from the noise term $U$ in Eq.~\eqref{eq:confounder_latent_noise}.
%

\subsection{Practical considerations}
\label{sec:latent_dml_practical}
As shown 
in Appendix~\ref{sec:latent_dml_theory}, under well-specified noise models, 
the causal effect estimate $\hat{\theta}$ asymptotically converges to the true causal effect
even when there are unobservable confounders.
This property is particularly beneficial in practice, as it allows for robust causal estimation while keeping residual estimation and latent variable modeling separate.
By decoupling these components, our approach reduces computational cost and improves stability, particularly in large-scale models~\citep{dnn_causal_2021}.
%
Moreover, it demonstrates that even in the absence of unobservable confounders, latent DML can reduce the variance of the causal effect estimator by leveraging additional latent structure.

%
%
%
%

Choosing an appropriate model for latent DML can be challenging due to limited domain knowledge.  
%
To address this, we propose using model selection criteria, such as the Bayesian Information Criterion (BIC) \cite{Bishop06_PRML}, to identify a model that best describes the data. 
The effectiveness of this approach will be empirically evaluated in Section \ref{sec:experiment_model_selection}.
Moreover, our approach remains valid even under model misspecification by approximating the underlying distribution \citep{huber1967behavior, white1982maximum}.

\section{Related work}
\label{sec:related_work}
Latent variable modeling in causal inference has been developed to address various challenges, such as the impact of outliers and unobserved confounders.
\citet{causal_robust_synthetic_control_2018,robust_causal_inference_2024} focus on the influence of outliers, aligning with our first scenario where a latent variable affects only the outcome.
Other works such as \citep{causal_vae_2017,lvm_root_causal_inference_2023,online_causal_structure_learning_latent_confounders_2023,causal_reasoning_latent_confounders_2023,cross_moment_causal_2023}
address unobserved confounders, related to our second scenario where a latent variable influences both treatment and outcome.
Notably, our use of discrete latent variables in the second scenario shares similarities with Latent Class Analysis (LCA), which identifies hidden subgroups within a population \citep{lca_2020,lca_developmental_research_2016}.
Despite these diverse approaches in latent variable modeling, our study is the first to integrate latent variable modeling into the DML framework, allowing for explicit estimation of latent structure within a two-stage causal inference setting.

Addressing unobserved confounders is a central challenge in causal inference.
Recent works like RieszNet and ForestRiesz adjust the feature space to improve robustness against hidden confounding, though they do not explicitly model latent variables \citep{Riesz_debiased_ml_2022, Omitted_variable_bias_2022, riesz_net_2020}.
%
A recent study
\citep{dr_causal_latent_factor_model_2024} leverages matrix completion for doubly robust estimation under unobserved confounding, exploiting multiple outcomes to identify low-dimensional latent factors. 
Feature selection techniques based on orthogonal search or doubly robust methods have also been explored to reduce sensitivity to unobserved confounders \citep{Causal_feature_selection_tmlr2023, DRCFS_icml2023}.
Our latent DML method is complementary to these approaches, as it directly models latent factors rather than merely mitigating their influence.
Integration with such methods may further help address unobserved confounding.

Numerous DML extensions have been proposed to handle various practical challenges, including
bias reduction \citep{coordinated_dml_2022}, 
conditional average treatment effects estimation \citep{dml_conditional_effect_2021},
dynamic treatment effect estimation \citep{dml_dynamic_effect_2021}, 
difference-in-differences \citep{dml_diff_in_diff_2020},
off-policy evaluation \citep{dml_ope_2020, dml_recom_2021},
kernel-based nonparametric inference \citep{dml_kernel_2020}, 
and computational efficiency improvements \citep{dml_parallel_compute_2021}.
Our work aligns most closely with \citet{dml_multiway_cluster_2021}, who address complex data structures such as multiway clustering within the DML setup.
In contrast, we introduce latent variable modeling into DML, enabling latent structure to be inferred directly from probabilistic models.
Our approach is orthogonal to many existing DML extensions and could potentially enhance their applicability to settings with latent structure.

Our approach is based on the likelihood framework for estimating residual distributions. It remains valid even under model misspecification, as we estimate models with finite-dimensional parameters to approximate the underlying distribution using a certain distance function. This argument has been widely used throughout statistics and econometrics \citep{beran1977minimum, buja2019models1, buja2019models2, huber1967behavior, rakhlin2017empirical, rinaldo2010generalized, white1980using, white1982maximum}. The recent literature on causal inference also applies the same argument \citep{chernozhukov2018doubledebiased, cuellar2020propensity, kennedy2019robust, neugebauer2007inverse, dml_conditional_effect_2021, laan2006targeted}.


\begin{figure*}[htb]
  \centering
 \includegraphics[width=1.\linewidth]{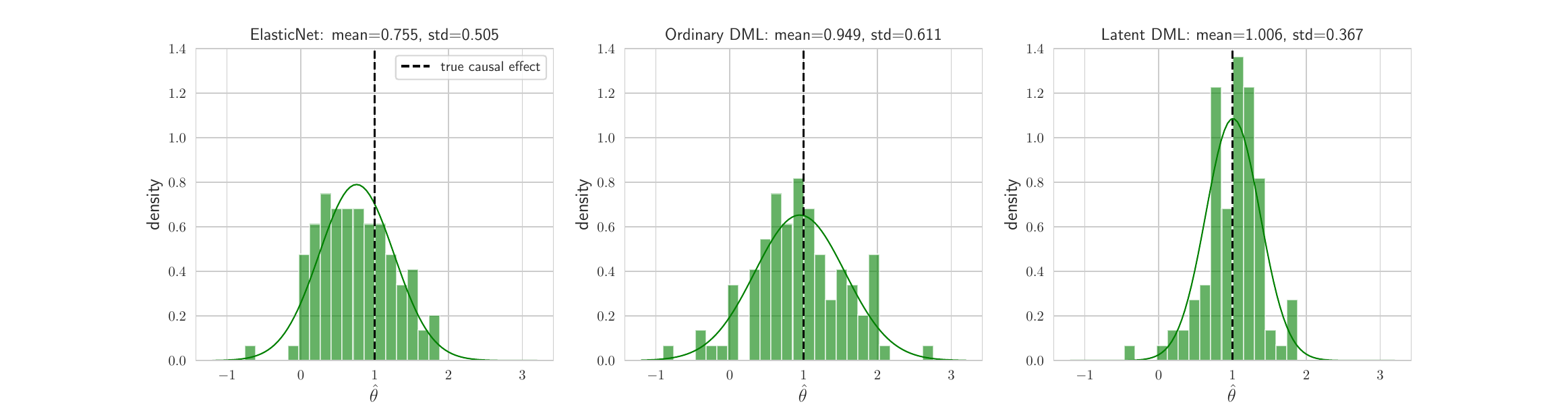}
 \vspace{-5mm}
 \Description{Outcome-latent setting: three histograms (ElasticNet, DML, latent DML); latent DML is closest to the true effect.}
 \caption{Distribution of causal effect estimates under the outcome latent variable across multiple runs. ElasticNet, ordinary DML, and outcome latent DML methods are compared, demonstrating the latter (ours) is closer alignment with the true causal effect.}
 \label{fig:output_only_latent}
\end{figure*}


\begin{figure*}[t!]
 \vspace{1.mm}
 \raggedright
 (A) Positive confounding ($a=2, b=2$)
 \\
 \begin{minipage}{\linewidth}
 \includegraphics[width=1.\linewidth]{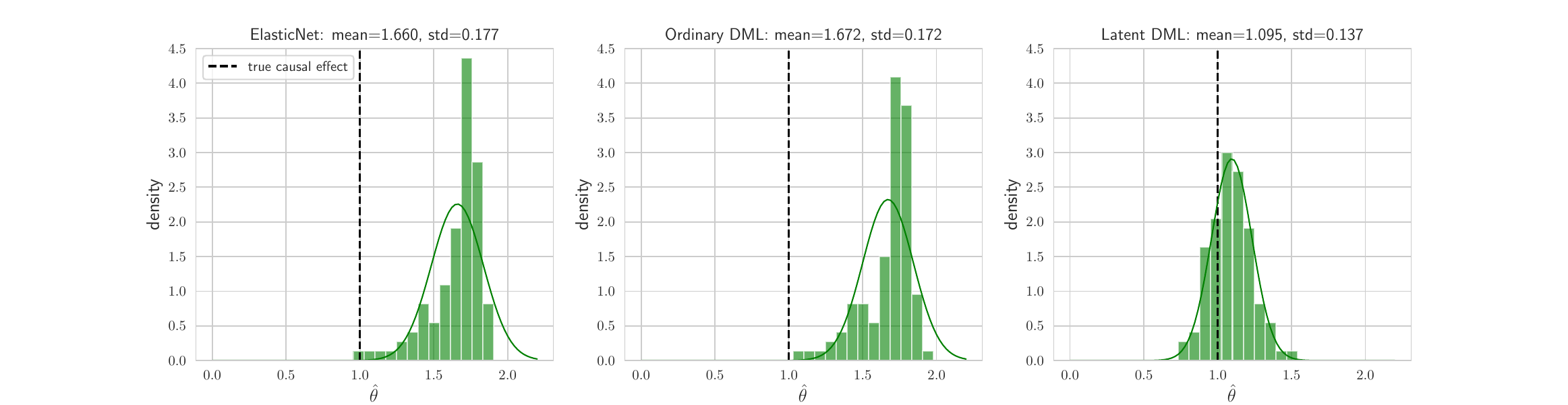}
 \vspace{-1.5mm}
 \end{minipage}
 (B) Negative confounding ($a=2, b=-2$) 
 \\
 \begin{minipage}{\linewidth}
 \includegraphics[width=1.\linewidth]{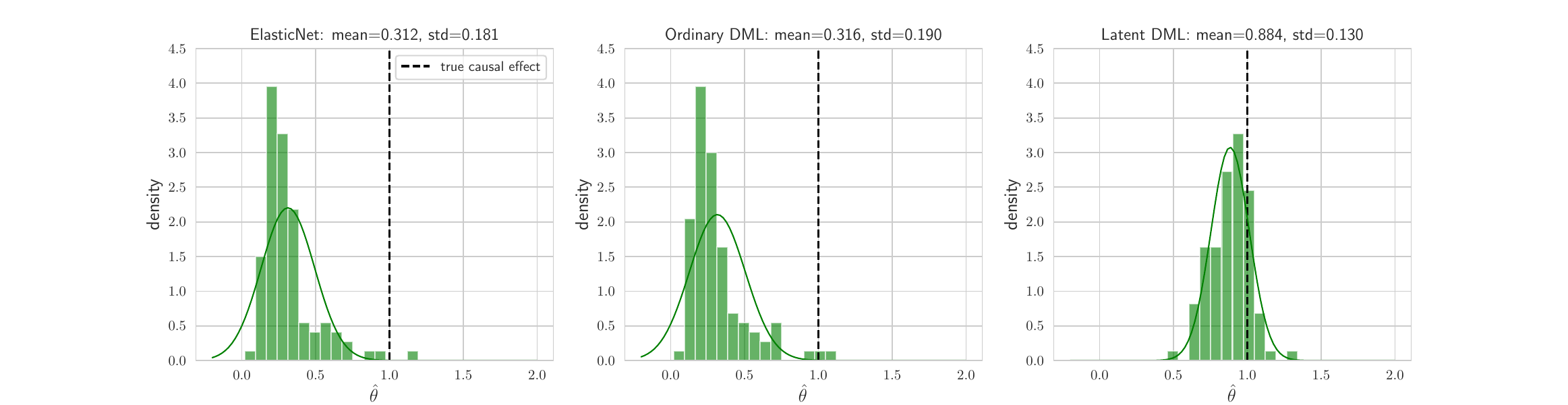}
 \end{minipage}
 \Description{Confounder-latent setting: two panels—(A) positive confounding and (B) negative; latent DML aligns best with the true effect in both.}
 \vspace{-2mm}
 \caption{Causal effect estimates in scenarios with unobserved confounders across multiple runs. (A) shows the case with positive confounding, and (B) shows negative confounding. Each subplot compares ElasticNet, ordinary DML, and confounder latent DML methods, highlighting the ability of the proposed method to estimate unbiased causal effects.}
 \label{fig:unobservable_confounder}
 \vspace{-1mm}
\end{figure*}

\section{Experiments}
\label{sec:experiment}

In this section, we conduct a series of numerical experiments. 
Since the true causal effects in real data are generally unknown, 
we first perform comprehensive experiments using synthetically generated data in Section \ref{sec:experiment_synthetic}. 
Experiments using real data are also presented in Section \ref{sec:experiment-real}.

\subsection{Synthetic data}
\label{sec:experiment_synthetic}
We first compare and evaluate the latent DML methods against the ordinary DML and regression analysis methods.
The focus here is to specifically assess the enhancements achieved by integrating latent variables into the DML framework, 
rather than conducting an extensive comparison with a wide array of baseline methods.
Additionally, experimental results regarding the selection of latent variable models are also presented.

The experimental setup with synthetic data was designed to rigorously evaluate the outcome latent and confounder latent models
in the latent DML framework under different conditions.
We conducted 100 runs with $300$ samples and $100$ covariates in each evaluation,
where covariates, treatment, and outcome are generated randomly with 
some fixed parameters, ensuring consistent experimental conditions.
In the first step of DML, ElasticNet \cite{Hastie09a} was employed for computing the residuals.
We also evaluated ElasticNet as a baseline, where the coefficient corresponding to the treatment variable served as the causal effect estimate.
For more details on the experimental setup, see Appendix \ref{app:implementation} and \ref{app:synthetic}. 

\subsubsection{Causal effect estimation under latent variables}
In our experiments, we first evaluated the outcome latent model as detailed in Section \ref{sec:latent-dml-outcome-only},
performing tests in an ideal setting without model misspecification.
The noise in this setup adhered to the pattern outlined in Eq.~\eqref{eq:z_expon}.
As shown in Figure \ref{fig:output_only_latent}, the outcome latent DML method proved effective in estimating causal effects, exhibiting minimal bias and variance.

We then expanded our investigation to scenarios involving unobserved confounders, maintaining the ideal test conditions.
This phase included testing cases where unobserved confounders influenced the outcome and treatment in both similar and opposite directions.
Figure \ref{fig:unobservable_confounder} shows that only the confounder latent DML method was able to accurately estimate the unbiased causal effect in these settings.

\subsubsection{Model selection}
\label{sec:experiment_model_selection}
The robustness of the latent DML framework against model misspecification was a key focus of our investigation. 
%
We evaluated five methods: ElasticNet, ordinary DML, outcome latent DML, confounder latent DML, and the BIC-selected latent DML.
The BIC-selected latent DML method chooses either the outcome or confounder latent DML variant based on the lowest BIC score.

Data generation encompassed four different scenarios, each with distinct noise structures and parameters. 
This included conditions without latent variable noise as well as those with Laplace-distributed output noise, 
the latter representing a case of model misspecification. 
The results in Figure \ref{fig:model_selection} demonstrate the effectiveness of our approach.
They indicate that the latent DML framework is robust, facilitating reliable model selection based on BIC,
even in the absence of prior knowledge and amidst the challenges posed by model misspecification.

\begin{figure}[htbp]
 \vspace{-0.5mm}
 \centering
 \includegraphics[width=0.9\linewidth]{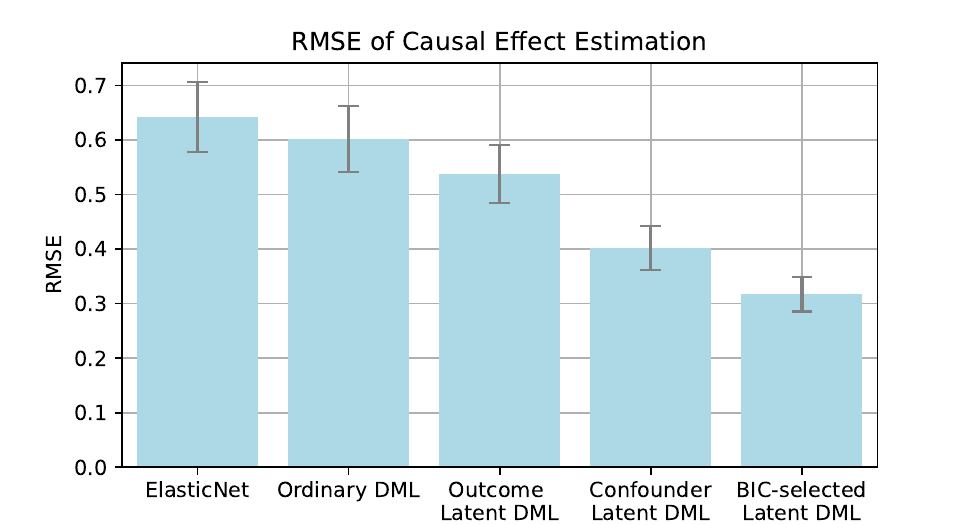}
 \vspace{-1mm}
 \Description{Bar plot comparing causal effect estimation errors. Latent DML with BIC selection achieves the lowest RMSE with error bars.}
 \caption{Causal effect estimation across diverse experimental settings, including model misspecification scenarios. Error bars indicate standard errors.} 
 \label{fig:model_selection}
\vspace{-1mm}
\end{figure}

\subsection{Real data applications}
\label{sec:experiment-real}
%
We conducted experiments on two real-world datasets: 
(i) the 401(k) dataset~\cite{401k_2003,chernozhukov2018doubledebiased}, which examines the effect of 401(k) participation on financial wealth (Section~\ref{sec:experiment-401k}), and 
(ii) an analysis of advertising operations, specifically investigating how the timing of advertisements impacts Return on Advertising Spend (ROAS) (Section~\ref{sec:experiment-ad}).

\subsubsection{401(k) participation and financial wealth}
\label{sec:experiment-401k}

We evaluated our method on the well-studied 401(k) dataset~\cite{401k_2003,chernozhukov2018doubledebiased,doubleml_401k}, which contains 9,915 household-level records from the 1991 Survey of Income and Program Participation (SIPP).
The outcome variable in our analysis was \texttt{net\_tfa}, representing net financial assets, and the treatment variable was \texttt{p401}, indicating participation in a 401(k) plan. 
The dataset included 10 covariates capturing demographic and financial characteristics, such as age, income, education, marital status, and home ownership.  

To assess robustness against missing covariates, we considered two settings where a key variable was unobserved:
\begin{itemize}
    \item \texttt{e401} – indicating eligibility for a 401(k) plan, highly predictive of treatment but less directly linked to the outcome.
    \item \texttt{pira} – indicating participation in an IRA, which likely correlates with both treatment and outcome via financial planning behavior.
\end{itemize}
These scenarios emulated latent confounding or selection bias.

We performed Monte Carlo resampling by drawing 2,000-sample subsets from the full dataset and repeated this process 100 times. As a reference, the causal effect of 401(k) participation was estimated at \$14.75K using ordinary DML on the full dataset with all covariates. Bias was reported relative to this benchmark.

We compared ordinary DML, Coordinated DML~\cite{coordinated_dml_2022}, and our latent DML method using BIC-based model selection. Table~\ref{tab:401k_results} summarizes the results.

As shown in Table \ref{tab:401k_results}, latent DML consistently achieved lower bias than standard and coordinated DML when important covariates were omitted. The improvement was particularly notable when \texttt{pira}, a proxy for financial sophistication, was missing—suggesting that latent DML effectively adjusted for unobserved confounders that influenced both treatment and outcome. In contrast, Coordinated DML partially mitigated bias when the omitted variable was only related to treatment (e.g., \texttt{e401}), but fell short under confounding scenarios. These results illustrate the robustness of our method against latent confounding structures.

Further implementation details on Coordinated DML are provided in Appendix~\ref{app:real_data_application_401k}.

\begin{table}[ht]
\caption{Causal effect estimation using the 401(k) dataset with Monte Carlo resampling. 
Bias values are reported as ``bias ± 95\% CI width,'' with a reference causal effect 
set at \$14.75K based on the DML estimate on full data without missing covariates.  
All values in the table are in thousands of dollars. 
\textsc{Latent DML (Ours)} denotes the BIC-selected latent DML method.
}
\vskip -0.06in
\label{tab:401k_results}
\begin{center}
\begin{small}
\begin{tabular}{llc}
\toprule
Covariates & Method & Bias  \\
\midrule
Full        & DML & 0.05 $\pm$ 0.86 \\ 
Full        & Coordinated DML & -0.79 $\pm$ 1.26 \\ 
Full        & Latent DML (Ours) & -0.76 $\pm$ 0.82  \\ 
\midrule
w/o \texttt{e401} & DML & -2.96 $\pm$ 0.70 \\ 
w/o \texttt{e401} & Coordinated DML & -1.84 $\pm$ 0.89 \\ 
w/o \texttt{e401} & Latent DML (Ours) & -1.64 $\pm$ 1.85  \\
\midrule
w/o \texttt{pira} & DML & 2.66 $\pm$ 0.89 \\ 
w/o \texttt{pira} & Coordinated DML & 1.69 $\pm$  1.17\\ 
w/o \texttt{pira} & Latent DML (Ours) & 1.22 $\pm$ 0.78  \\ 
\bottomrule
\end{tabular}
\end{small}
\end{center}
\vspace{-1mm}
\end{table}

\subsubsection{Advertisement timing and ROAS}
\label{sec:experiment-ad}
This experiment investigated the causal effect of advertisement timing on advertising effectiveness using real-world data from a domestic e-commerce platform. 
The treatment variable was defined as the number of days between when advertising becomes feasible (after sales exceed a threshold) and its actual commencement.

The dataset consisted of 155 cases and 45 covariates, including temporal measures, sales and advertising metrics, project-specific indicators, and curator variables (see Table~\ref{tab:real_data_apprication_covariates}).
These covariates were carefully selected to control for potential confounders and ensure a reliable causal analysis.
Advertising effectiveness was measured using Return on Advertising Spend (ROAS), a key performance indicator in digital marketing that evaluates how efficiently an advertising campaign converts spending into revenue.
ROAS was computed as the ratio of revenue generated from advertisements to the cost of those advertisements.

Table~\ref{tab:real_data_aplication_result} presents the results. 
The ordinary DML method produced unexpected findings, contrary to domain knowledge. 
Typically, a shorter delay in advertisement commencement is expected to increase ROAS, implying a negative causal effect. 
However, ordinary DML estimated a positive causal effect of +29.5\% per 100 days, contradicting this expectation. 
In contrast, Latent DML with model selection based on BIC estimated a negative causal effect of -60.4\% per 100 days, which aligns with prior expectations. 
This finding reinforces the validity of our proposed approach.
These results highlight the importance of accounting for latent factors that may influence both the timing of advertising and its effectiveness.

\begin{table}[htbp]
\centering
\caption{Experimental results: causal effect of advertisement timing on ROAS (percentage per 10 days).
Latent DML uses BIC to choose between outcome and confounder latent models. Based on this selection, the estimated causal effect of Latent DML is therefore $-60.37$.}
\label{tab:real_data_aplication_result}
\vspace{-1mm}
\begin{tabular}{lcc}
\hline
Method & Estimated Effect & BIC \\
\hline
Ordinary DML & 29.45 & 877.55 \\
Outcome Latent DML & -60.37 & 839.26 \\
Confounder Latent DML & 0.91 & 858.66 \\
\hline
\end{tabular}
\vspace{-1mm}
\end{table}

\section{Conclusion}
\label{sec:conclustion}

%
%
%

This paper introduces a novel framework that integrates latent variable modeling into causal effect estimation, leveraging the structure of double machine learning (DML) for computational and statistical efficiency.
The proposed latent DML framework accounts for unobserved factors that may influence treatment or outcome, addressing challenges posed by missing or unmeasured covariates.
By introducing latent variables only in the second stage of DML, our method decouples residual estimation from latent variable inference, improving tractability and scalability without sacrificing estimation consistency.
We demonstrate the effectiveness of the proposed approach through extensive empirical validation on both synthetic and real-world datasets, supported by theoretical analysis of consistency, underscoring its robustness and practical applicability.
Future directions include relaxing noise-model assumptions, analyzing model misspecification, and exploring variational inference methods as an alternative to the EM algorithm.

\appendix
\section{Appendix}
Appendix is organized as follows. 
Section \ref{app:experiment} provides additional details on the simulation study and empirical analysis.
Section \ref{sec:latent_dml_theory} presents theoretical results for our estimation procedure.
Further details and full proofs are provided in Appendix \ref{app:extended_appendix}.

\subsection{Experimental details}
\label{app:experiment}
This section provides additional details of the experimental setup. 
%

\subsubsection{Experimental setup and model training}
\label{app:implementation}
We describe implementation details of our experiments, including the cross-validation strategy, hyperparameter tuning, and training procedures.

In our experiments, we used ElasticNet both as a baseline method for causal effect estimation and as the ML model for the first step of the DML frameworks. 
ElasticNet is a linear regression model that combines both L1 and L2 regularization, with two hyperparameters: $\alpha$ and \texttt{l1\_ratio}. The hyperparameter $\alpha$ controls the overall strength of the regularization, while \texttt{l1\_ratio} determines the relative contribution of L1 and L2 regularization. Specifically, $\texttt{l1\_ratio}=1.0$ corresponds to Lasso regression, and $\texttt{l1\_ratio} = 0.0$ corresponds to Ridge regression. 
For both the baseline and DML-based models, these hyperparameters were selected using 5-fold cross-validation (CV). The search space for $\alpha$ was set to $\{10^{-2}, 10^{-1}, 1, 10, 100\}$, and \texttt{l1\_ratio} was varied over $\{0.0, 0.25, 0.5, 0.75, 1.0\}$. %
This approach allows for robust estimation by combining both types of regularization.

For the baseline ElasticNet method, we estimated the causal effect by directly using the coefficient corresponding to the treatment variable in the final fitted ElasticNet model.
Unlike DML, this approach does not employ sample splitting or residual estimation, making it a simple but useful benchmark for comparison.

For the DML frameworks, ElasticNet was employed in a slightly different manner.
Specifically, we used a 5-fold cross-fitting approach ($K=5$), where the model was trained on four folds, and residuals were computed on the remaining fold in a rotating manner over all five folds, ensuring that each sample was used for residual estimation once.
This process is detailed in Algorithm \ref{algo:dml} for DML and Algorithm \ref{algo:latent-dml} for latent DML.

\subsubsection{Implementation details for 401(k) dataset experiment}
\label{app:real_data_application_401k}

We used the official implementation available at \url{https://github.com/nitaifingerhut/C-DML} for Coordinated DML~\citep{coordinated_dml_2022}. 
This implementation differed slightly from ours in its treatment of sample splitting. Specifically, it used a two-split strategy, where the model was trained on one subset and residuals were computed on the other. 
In contrast, our DML and latent DML methods adopted a 5-fold cross-fitting approach, as described in Section~\ref{app:implementation}. 
While this difference may have caused some variations in estimation performance, we reported the Coordinated DML results without any modifications to maintain consistency with prior work.

\subsection{Theoretical analysis}
\label{sec:latent_dml_theory}

In this section, we present the theoretical properties of our framework, including the consistency and asymptotic normality of the latent DML estimator. 
The detailed identification analysis and all associated proofs, including those of the theorems stated below, are provided in Appendix \ref{app:extended_analysis}.

The population log-likelihood function $L(\gamma)$ is defined in Eq.~\eqref{eq:objective_lvm}  
over the parameter space $\Gamma \subseteq \mathbb{R}^{d_{\gamma}}$, with $d_{\gamma}$ denoting the number of elements of $\gamma$.
We denote the gradient by $\nabla_{\gamma} L(\gamma) := \partial L(\gamma) / \partial \gamma$, and let $\gamma_{0} \in \Gamma$ satisfy the first-order condition $\nabla_{\gamma} L(\gamma_0) = 0$.

We now present the consistency and asymptotic normality of the latent DML estimator. 
To derive these asymptotic properties, we express the score function (the first derivative of the log-likelihood function) by using the moment function $\psi(\cdot)$ as follows: 
\begin{equation*}
    \frac{1}{n} \sum_{i=1}^{n} \psi(W_{i}, \gamma, \nu), 
\end{equation*}
where 
$\{ W_{i} \}_{i=1}^{n}$ with $W_{i}:=(D_{i}, X_{i}, Y_{i})$ is the observed sample,
$\gamma$ represents the finite-dimensional parameters, including both causal effect and late model parameters, and $\nu$ denotes the possibly infinite-dimensional nuisance parameters related to the first step in the DML procedure, or $\nu = (h, m)$. In what follows, 
$\gamma_{0}$ and $\hat{\gamma}$ denote the true finite-dimensional parameter and its estimator, respectively, 
while 
$\nu_{0}$ and $\hat{\nu}$ denote the nuisance parameter and its estimator. 

The following conditions are imposed on the model space, observations, moment function:


\begin{assumption} \label{as:model}
The model space is well specified, meaning that the true parameter $\gamma_0$ belongs to  $\Gamma$, i.e., $\gamma_0 \in \Gamma$.
\end{assumption}

\begin{assumption}\label{as:iid}
    The sample $\{W_{i}\}_{i=1}^{n}$, where $W_{i}:=(D_{i}, X_{i}, Y_{i})$, is 
    independently and identically distributed (i.i.d.) from $W:=(D,X,Y)$.
\end{assumption}

\begin{assumption}\label{as:1/4}
  The estimator $\hat{\nu}$ of the nuisance function $\nu_{0}$ is $n^{1/4}$-consistent in the squared mean-square-error sense, i.e.
  \begin{equation}
    \Big ( 
    \E_{X} 
    \Big [
      \big \|
      \hat{\nu}(X)-\nu_0(X)
      \big \|^2
      \Big] 
      \Big )^{1/2}
      = o_p(n^{-1/4}) ,
  \end{equation}
  The nuisance function and its estimator 
  are uniformly bounded by a constant. 
  That is,  
  $\|\nu_0(x)\| \leq C$
  and 
  $\|\hat{\nu}(x)\| \le C$    
  uniformly in $x\in \mathcal{X}$ for some constant $C$ almost surely.
\end{assumption}

\begin{assumption}\label{as:moment}
  The following conditions hold for the moment function $\psi$:
  \begin{enumerate}
    \item[(a)] 
    The moment function 
    satisfies that 
    $\E\big [\psi \big(W, \gamma, \nu_{0}(X) \big) \big] = 0 $
    only when $\gamma = \gamma_0$. 
    \item[(b)]
    The moment function 
    $\psi \big(W, \gamma, \nu_{0}(X) \big) $
    is dominated by some integral function.  
    \item[(c)] 
    The function $\gamma \mapsto \psi(w, \gamma, \nu)$ is continuously differentiable for any $w, \nu$, and its derivative is bounded by an integrable function $c(w)$, i.e., $\|\nabla_\gamma \psi(w, \gamma, \nu)\| \leq c(w)$ with $\E[|c(W)|] < \infty$. $\E\left[\nabla_\gamma \psi(W,\gamma_0,\nu_0)\right]$ is non-singular.
  \item[(d)] 
  The function $\nu \mapsto \psi(w, \gamma, \nu)$ is twice differentiable for any $w, \gamma$,
   The spectral norm of $\nabla_\nu \psi(w,\gamma,\nu)$ is uniformly bounded by $\sigma$. That is,
  $\|\nabla_\nu \psi(w,\gamma,\nu)\|_{op} \leq \sigma$ for all $w, \gamma, \nu$.
  the Hessian $\nabla_{\nu\nu} \psi(z,\gamma,\nu)$ has the largest eigenvalue bounded by some constant $\lambda_{\max}$ uniformly for all $\gamma$ and $\nu$. 
  \end{enumerate}
  \end{assumption}

These regularity conditions above are standard in the literature. See \citet{dml_2018} for instance.

To establish the consistency and asymptotic normality of the finite-dimensional parameters when the nuisance parameters satisfy Assumption \ref{as:1/4}, we employ Neyman orthogonality, which ensures the moment condition is insensitive to perturbations in the nuisance parameters. Formally, this condition is expressed as
\begin{align}
    \label{eq:ny}
    \frac{\partial}{\partial s}
    \E[
    \psi(W,  \gamma , \nu + s  \xi)
    ] \big |_{s=0}
    = 0, 
\end{align}
for all $\xi$ in an appropriate space and any real value $s$.
The following lemma establishes the Neyman orthogonality for the key moment function of our proposed estimator under the specification considered in this paper. 
While Neyman orthogonality may not immediately hold for more general specifications of latent distributions, the score function can be reformulated to satisfy this condition, as shown in 
\citet{morzywolek2023general}
and 
\citet{battey2024role} among others.

In the lemma below, 
we consider the specific moment function, given as 
\begin{align*}
    \phi(W,  \gamma, \nu)
    := &
    \big (
    Y - h(X) 
    - \theta 
    \big( D - m(X) \big) 
    - q_{\gamma}(W)
    \big ) \\
    & \times 
    \big( D - m(X)  - r_{\gamma}(W)\big), 
\end{align*}
where $\gamma$ represents all finite-dimensional parameters including causal effect parameter $\theta$ as well as the latent variable model.
$q_\gamma(\cdot)$ and $r_\gamma(\cdot)$ are model-specific real-valued functions parametrized by $\gamma$.
$h(\cdot)$ and $m(\cdot)$ represent nuisance functions in the double machine learning procedure, i.e., $h, m \in \nu$.
In our latent DML scenario in Eq.~\eqref{eq:score-function}, we set $q_{\gamma}(W) = \mathbb{E}[{\yltFun}(Z) \mid W, \gamma]$ and $r_\gamma(W) = 0$.

The moment function $\phi(\cdot)$ is assumed to satisfy Assumption \ref{as:moment}.
Additionally, the conditional mean-zero restrictions hold:
\begin{eqnarray}
\label{eq:mean-zero}
    \E
    \big [
    Y - h(X) 
    - \theta 
    \big( D - m(X) \big) 
    - q_{\gamma}(W)
    \big |X, D
    \big ] & = 0 \\
    \E
    \big [
     D - m(X)  - r_{\gamma}(W)
    \big | X
    \big ] &= 0.  \label{eq:mean-zero-1}
\end{eqnarray}


\begin{lemma}[Neyman Orthogonality] 
\label{pro:Neyman}
Suppose that Assumption \ref{as:moment} and the condition in (\ref{eq:mean-zero})-(\ref{eq:mean-zero-1}) hold. Then, we have 
\begin{align*}
    \frac{\partial}{\partial s}
    \E[
    \phi(W,  \gamma , \nu + s  \xi)
    ] \big |_{s=0}
    = 0, 
\end{align*}
for all $\xi$ in an appropriate space with a real value $s$. 
\end{lemma}

We now formally establish the consistency of our estimator under the regularity conditions in the following theorem.

\begin{theorem}[Consistency]
\label{theorem:consistency}
Under Assumptions \ref{as:model}-\ref{as:moment} and the Neyman orthogonality condition in (\ref{eq:ny}), we have, as $n \to \infty$,
\begin{equation*}
\hat{\gamma} \to^p \gamma_0. 
\end{equation*}
\end{theorem}

We next turn to the asymptotic distribution of our estimator and establish its limiting behavior. The following theorem shows that our estimator converges to a normal distribution and provides its asymptotic variance.

\begin{theorem}[Asymptotic Normality]
\label{theorem:normality}
    Under Assumptions \ref{as:model}-\ref{as:moment} and the Neyman orthogonality condition in (\ref{eq:ny}), we have, as $n \to \infty$,
    \begin{eqnarray*}
        \sqrt{n}(\hat{\gamma}-\gamma_0)  \to^d N(0, \Omega),
    \end{eqnarray*}
    where 
    $\Omega:=
    H
    Var\big ( \psi(W,\gamma_0, \nu_0) \big )
    H^{\top}
     $
     with 
     $
     H:=
      (\E [\nabla_\gamma \psi(W,\gamma_0, \nu_0)])^{-1}.
     $
\end{theorem}

The asymptotic variance derived in Theorem \ref{theorem:normality} demonstrates the potential benefit of controlling for unobserved variables even in the absence of confounding effects. This advantage stems fundamentally from conditioning on additional information, which reduces variability in estimating equations. More precisely, by the law of total variance, the conditioning process necessarily decreases the overall variance of our estimating equations, which translates directly to smaller asymptotic variance and consequently narrower confidence intervals. This efficiency gain occurs despite the absence of endogeneity issue, underscoring that control variables can enhance estimation precision beyond their conventional role in bias reduction.

\section*{Usage of generative AI}
We used generative AI tools at specific stages of this work. During manuscript preparation, OpenAI ChatGPT (GPT-4) was employed to revise phrasing and improve clarity. GitHub Copilot was used as an assistive tool during code development to suggest code snippets. These tools were not used to generate scientific content, design algorithms, or create experimental results. All core ideas, model development, implementation, and analyses were conducted and validated by the authors.

\bibliographystyle{ACM-Reference-Format}
\bibliography{dml.bib}


\onecolumn

\appendix
\setcounter{section}{1}

\section{Extended Appendix}
\label{app:extended_appendix}
We provide additional information to complement the main paper, including experimental settings and theoretical analysis.

\subsection{Additional experimental settings}
We describe additional experimental settings in Section~\ref{app:synthetic} for the synthetic experiments presented in Section~\ref{sec:experiment_synthetic} of the main paper.
Details on the compute resources used are provided in Section~\ref{app:experiment-resource}.
We also include the full list of covariates used in the real data analysis of advertisement timing (Table~\ref{tab:real_data_apprication_covariates}).

\subsubsection{Synthetic data}
\label{app:synthetic}
Our experiments utilized a synthetic data generation process that closely simulates real-world complexities. 
We modeled nuisance components $m$ and $g$ as sparse linear models with $100$ covariates, 
using coefficients drawn from a normal distribution to reflect real-world data sparsity.
The noise distributions basically follow Eqs.~\eqref{eq:u_expon} and \eqref{eq:confounder_latent_noise}.
The data generator was configured to create datasets with specific attributes such as sample size $300$, 
number of covariates $100$, true causal effect $1.0$.
We used this setup to conduct two sets of experiments: 
one verifying the outcome latent DML and another focusing on the confounder latent DML method.

For the first experiment, the noise $U$ in the outcome includes an additional noise generated by the exponential distribution with a mean of $5.0$. For the second, we examined cases with both positive and negative correlations between confounders and outcomes, where the confounders are discrete and generated as Bernoulli-distributed variables with a randomly specified success rate $q$.
Specifically, in the positive correlation case, the parameters $a$ and $b$ in Eq.~\eqref{eq:confounder_latent_noise} were both set to $2.0$, while in the negative correlation case, they were set to $2.0$ and $-2.0$, respectively.
In both experiments, the standard deviations of the outcome and treatment noises were set to $1.0$ and $0.5$, respectively.
This comprehensive approach allowed for an in-depth evaluation of our latent DML framework, demonstrating its capacity to capture causal effects accurately in a range of simulated environments.

\subsubsection{Compute Resources}
\label{app:experiment-resource}

All experiments were conducted on a standard local machine with 2.3GHz Quad-Core Intel Core i7 CPU and 32 GB Memory.
The computational time for each experimental run varied between 2 to 5 hours depending on the complexity of the experiment.
Given the relatively modest computational requirements, all experiments were feasible to run on this local machine without the need for specialized hardware or cloud resources.
\newpage

\begin{table*}[ht]
\centering
\caption{List of covariates for real data application: causal effect analysis of advertisement timing to ROAS}
\label{tab:real_data_apprication_covariates}
\begin{tabular}{ll}
\hline
\textbf{Variable Name}                          & \textbf{Description}                                        \\ \hline
DaysPostThresholdToAdStart                      & Days from surpassing a sales threshold to starting advertising \\
DaysToThresholdForAd                            & Days to reach a sales threshold for advertisements to be OK\\
AdStartsWithinEarlySalesPhase                   & Indicator if advertisement starts within the early sales phase \\
AdEndsWithinLateSalesPhase                      & Indicator if advertisement ends within the late sales phase \\
TotalDaysOfAd                                   & Total days of running advertisements \\
LogSystemAvgSales                               & Log of average sales in the system \\
LogSystemStdDevSales                            & Log of sales standard deviation in the system \\
LogSalesInitialPeriod                           & Log of sales in the initial period \\
LogSalesRatioMidToInitialPeriod                 & Log of sales ratio from middle to initial period \\
LogUnitsSoldInitialPeriod                       & Log of units sold in the initial period \\
LogUnitsSoldRatioMidToInitialPeriod             & Log of units sold ratio from middle to initial period \\
LogPreAdAvgSalesRatioOverInitialPeriod          & Log of average sales ratio before ad start over the initial period \\
PreAdStartSalesSlope                            & Relative sales slope before advertisement start \\
PreAdEndSalesSlope                              & Relative sales slope before advertisement end \\
LogPreAdStartUnitsSoldRatioOverInitialPeriod    & Log of average units sold ratio before ad start over the initial period \\
LogPreAdEndUnitsSoldRatioOverInitialPeriod      & Log of average units sold ratio before ad end over the initial period \\
PreAdStartUnitsSoldSlope                        & Relative units sold slope before advertisement start \\
PreAdEndUnitsSoldSlope                          & Relative units sold slope before advertisement end \\
HolidayCountDuringProject                       & Number of holidays during the project period \\
LogTotalSales                                   & Log of total sales \\
LogSalesPostInitialPeriod                       & Log of sales after the initial period \\
AvgLogReturnAmount                              & Average of Log of return amount \\
StdDevLogReturnAmount                           & Standard deviation of Log of return amount \\
LogNumberOfReturns                              & Log of number of returns \\
ProjectDuration                                 & Duration of the project \\
NumberOfAdRuns                                  & Number of times advertisements were run \\
StartPeriod\_Q1                                 & Indicator if project started in Q1 \\
StartPeriod\_Q2                                 & Indicator if project started in Q2 \\
StartPeriod\_Q3                                 & Indicator if project started in Q3 \\
ProjectStartDay\_Holiday                        & Indicator if project started on a holiday \\
ProjectStartDay\_PreHoliday                     & Indicator if project started on a day before a holiday \\
Curator\_A                                      & Indicator for Curator A \\
Curator\_B                                      & Indicator for Curator B \\
Curator\_C                                      & Indicator for Curator C \\
Curator\_D                                      & Indicator for Curator D \\
Curator\_E                                      & Indicator for Curator E \\
Curator\_F                                      & Indicator for Curator F \\
\hline
\end{tabular}
\end{table*}
\newpage

\subsection{Theoretical analysis details}
\label{app:extended_analysis}
In Section~\ref{app:identification}, we establish identification conditions for our latent variable models. 
Section~\ref{app:proofs} provides the proofs of the theoretical results stated in Section~\ref{sec:latent_dml_theory} of the main paper.

\subsubsection{Identification of latent variable models}
\label{app:identification}
We discuss the conditions necessary for the identification of latent variable models within latent DML.
The population log likelihood function $L(\gamma)$is defined in Eq.~\eqref{eq:objective_lvm}  
over the parameter space $\Gamma \subseteq \mathbb{R}^{d_{\gamma}}$, with $d_{\gamma}$ denoting the number of elements of $\gamma$.
We denote by the derivative $\nabla_{\gamma} L (\gamma) := \partial L(\gamma) / \partial \gamma$ and let $\gamma_{0} \in \Gamma$ satisfy the first-order condition that 
$\nabla_{\gamma} L ( \gamma_{0}) = 0$.

We assume the following conditions.  

\begin{assumption} \label{as:as1}\ 
\begin{enumerate}
    \item[(a)] 
    $\Gamma$ is an open, convex set in $\mathbb{R}^{d_{\gamma}}$.
    \item[(b)]
    The map 
     $\gamma \mapsto L(\gamma)$ is continuously differentiable over $\Gamma$.
    \item[(c)] 
     The Fisher information $\mathcal{I}(\gamma):=\E[\nabla_{\gamma} L(\gamma) \nabla_{\gamma} L (\gamma) ^{\top}]$ is a matrix with all finite elements and is a continuous function for every $\gamma \in \Gamma$.
\end{enumerate}
\end{assumption}
%
The above assumptions are standard in the literature and commonly employed in theoretical analysis.
Under Assumption \ref{as:as1}, the proposition below shows the identification condition for $\gamma_{0}$.

%
\begin{proposition}[Identifiability]
\label{prop:Identifiability} 
Under Assumption \ref{as:as1}, 
 $\gamma^{0}$ is locally identified if and only if $\mathcal{I}(\gamma^\ast)$ is non-singular.
\end{proposition}
%
\begin{proof}[\hspace{-3.5mm}\textbf{Proof}]
Applying the standard identification result for the likelihood estimator, we can establish the desired result. See Theorem 1 of \citet{rothenberg1971identification}.
\end{proof}
The verification of Assumption \ref{as:as1} and the identifying restriction in the above proposition is straightforward for the models discussed in 
Section \ref{sec:latent-dml-outcome-only} and \ref{sec:latent-dml-confounder}. 

\subsubsection{Proofs}
\label{app:proofs}~
\\[1mm]
\hspace{-3.5mm}
\textbf{Lemma \ref{pro:Neyman}}\ (Neyman Orthogonality)
\textit{
Suppose that Assumption \ref{as:moment} and the condition in (\ref{eq:mean-zero})-(\ref{eq:mean-zero-1}) hold. Then, we have 
\begin{align*}
    \frac{\partial}{\partial s}
    \E[
    \phi(W,  \gamma , \nu + s  \xi)
    ] \big |_{s=0}
    = 0, 
\end{align*}
for all $\xi$ in an appropriate space with a real value $s$. }

\begin{proof}[\hspace{-3.5mm}\textbf{Proof}]
    Let $s\in\mathbb{R}$ be an arbitrary scalar value and 
    set $\xi:=(\xi_{h}, \xi_{m})^{\top}$ denote arbitrary deviations 
    from $\nu$, so that 
    $h(X) + s \xi_{h}(X)$
    and 
    $m(X) + s \xi_{m}(X)$
    are well-defined.
    Under Assumption \ref{as:moment}(b), we can show that 
    $\phi(W,\gamma, \nu + s\xi)$
    is dominated by an integrable function.  
    Using the dominated convergence theorem, we can show 
    \begin{align*}
       &  \frac{\partial}{\partial s}
       \E\big [
        \phi(W,  \gamma. \nu + s\xi)     
       \big]
       =
       \E\bigg [
         \frac{\partial}{\partial s}
        \phi(W, \gamma, \nu + s \xi)     
       \bigg]. 
    \end{align*}
    We rewrite the moment function as 
    $
        \phi(W, \gamma, \nu + s\xi)  
        = 
        \phi_{1}(W,  \gamma, \nu + s\xi)  
        \cdot 
        \phi_{2}(W, \gamma, \nu + s\xi)  , 
    $
    where 
    \begin{align*}
        \phi_{1}(W,  \gamma, \nu + s\xi) 
        &:= 
        Y - 
        \big ( h(X) + s \xi_{h}(X)  \big ) 
        - \theta 
        \big( D - m(X) - s\xi_{m}(X)\big) 
        - 
        q_{\gamma}(W)  ,\\
        \phi_{2}(W,  \gamma, \nu + s\xi)
        &:= 
         D - m(X)- s\xi_{m}(X) - r_{\gamma}(W).  
    \end{align*}
    Taking the derivative with respect to $s$, we have 
    \begin{align*}
        \frac{\partial}{\partial s}
        \phi(W,  \gamma, \nu + s\xi)   
        & = 
       \phi_{1}'(W, \gamma, \nu + s\xi)  
        \cdot 
        \phi_{2}(W, \gamma, \nu + s\xi) 
         +
       \phi_{1}(W,  \gamma, \nu + s\xi)  
        \cdot 
        \phi_{2}'(W, \gamma, \nu + s\xi) ,
    \end{align*}
    where 
    $\phi_{1}'(\cdot)$ and $\phi_{2}'(\cdot)$ denote the partial derivatives with respect to $s$,
    specifically given by  
    \begin{align*}
        \phi_{1}'(W,  \gamma, \nu + s\xi)
        &:=
         - \xi_{h}(X) 
         - 
          \theta 
          \xi_{m}(X) 
       \ \ \    \mathrm{and} \ \ \
        \phi_{2}'(W, \gamma,\nu + s\xi)
        := 
         -  \xi_{m}(X)  .  
    \end{align*}
    Also, it follows from (\ref{eq:mean-zero}) that 
    \begin{eqnarray*}
       \E[ \phi_{1}(W, \gamma, \nu + s\xi)| X, D] \big |_{s=0}  = 0
        \ \ \ \mathrm{and} \ \ \ \
       \E[ \phi_{2}(W, \gamma, \nu + s\xi)|X] \big |_{s=0} = 0.         
    \end{eqnarray*}
    Using the law of the iterated expectation, we can obtain the desired result. 
\end{proof}

\hspace{-3.5mm}\textbf{Theorem \ref{theorem:consistency}}\ (Consistency)
\textit{
Under Assumptions \ref{as:model}-\ref{as:moment} and the Neyman orthogonality condition in (\ref{eq:ny}), we have, as $n \to \infty$,
\begin{equation*}
\hat{\gamma} \to^p \gamma_0. 
\end{equation*}}
\begin{proof}[\hspace{-3.5mm}\textbf{Proof}]
By applying a first-order Taylor expansion to the empirical moment condition, 
$n^{-1}\sum_{i=1}^n \psi(W_i, \hat{\gamma}, \hat{\nu}) = 0
$, 
with respect to $\gamma$
around $\gamma_0$ with the mean value theorem, we obtain:
\begin{equation}
\label{eq:A1}
\frac{1}{n}\sum_{i=1}^n \psi(W_i, \gamma_0, \hat{\nu}) + \frac{1}{n}\sum_{i=1}^n \nabla_\gamma \psi(W_i, \tilde{\gamma}, \hat{\nu})(\hat{\gamma} - \gamma_0) =0, 
\end{equation}
where $\tilde{\gamma}$ represents a convex combination of $\gamma_0$ and $\hat{\gamma}$. 

For the first term, we apply another first-order Taylor expansion 
with respect to $\nu$ around $\nu_0$:
\begin{align*}
\frac{1}{n}\sum_{i=1}^n \psi(W_i, \gamma_0, \hat{\nu}) &= \frac{1}{n}\sum_{i=1}^n \psi(W_i, \gamma_0, \nu_0) + \frac{1}{n}\sum_{i=1}^n \nabla_\nu \psi(W_i, \gamma_0, \nu_0)(\tilde{\nu} - \nu_0) ,
\end{align*}
where $\tilde{\nu}$ represents a convex combination of $\nu_0$ and $\hat{\nu}$.
Here, the first term converges to $\E[\psi(W, \gamma_0, \nu_0)] = 0$ by the law of large numbers, while the second term converges to zero in probability by the Neyman orthogonality condition. 
It follows from (\ref{eq:A1}) that 
\begin{equation*}
\frac{1}{n}\sum_{i=1}^n \nabla_\gamma \psi(W_i, \tilde{\gamma}, \hat{\nu})(\hat{\gamma} - \gamma_0) = o_{P}(1).  
\end{equation*}
By the regularity conditions, we also have uniform convergence:
\begin{equation*}
\sup_{\gamma\in \Gamma} \left\|\frac{1}{n} \sum_{i=1}^n \nabla_\gamma \psi(W_i, \gamma, \hat{\nu})- \E[\nabla_\gamma \psi(W, \gamma, \nu_0)]\right\| \rightarrow^p 0.
\end{equation*}

Combined with the identification condition that $\E[\psi(W,\gamma,\nu_0)] = 0$ if and only if $\gamma = \gamma_0$, and the non-singularity of $\E[\nabla_\gamma \psi(W, \gamma_0, \nu_0)]$, 
the result that 
$ 
\E\big [\nabla_\gamma \psi(W, \gamma_0, \nu_0)\big] 
(\hat{\gamma} - \gamma_0) = o_p(1). 
$ leads to the desired conclusion:  
$\hat{\gamma} \to^p \gamma_0$.
\end{proof}

\hspace{-3.5mm}\textbf{Theorem \ref{pro:Neyman}}\ (Asymptotic Normality)
\textit{
Under Assumptions \ref{as:model}-\ref{as:moment} and the Neyman orthogonality condition in (\ref{eq:ny}), we have, as $n \to \infty$,
\begin{eqnarray*}
\sqrt{n}(\hat{\gamma}-\gamma_0)  \to^d N(0, \Omega),
\end{eqnarray*}
where 
$\Omega:=
H
Var\big ( \psi(W,\gamma_0, \nu_0) \big )
H^{\top}
 $
 with 
 $
 H:=
  (\E [\nabla_\gamma \psi(W,\gamma_0, \nu_0)])^{-1}.
 $
}
\begin{proof}[\hspace{-3.5mm}\textbf{Proof}]
  By applying a first-order Taylor expansion to the empirical moment condition with respect to $\nu$ around $\nu_0$, coupled with the mean value theorem, we obtain:
\begin{equation*}
\sqrt{n}(\hat{\gamma}-\gamma_0) = 
\left(  \frac{1}{n} \sum_{i=1}^n
\nabla_\gamma \psi(W_{i}, \tilde{\gamma}, \hat{\nu})
\right)^{-1} \frac{1}{\sqrt{n}} \sum_{i=1}^n \psi(W_{i},\gamma_0, \hat{\nu}),
\end{equation*}
where $\tilde{\gamma}$ represents some convex combination of $\gamma_0$ and $\hat{\gamma}$. 

We shall first show that the first term converges in probability to some positive-definite matrix and then establish that the second term converges in distribution to a multivariate normal distribution. The desired result follows from Slutsky's theorem.

Under the regularity conditions on the moments, we obtain a uniform law of large numbers for 
$(1/n) \sum_{i=1}^n
\nabla_\gamma \psi(W_{i}, \gamma, \hat{\nu})$, namely:
\begin{equation*}
\sup_{\gamma\in \Gamma} \left|\frac{1}{n} \sum_{i=1}^n
\nabla_\gamma \psi(W_{i}, \gamma, \nu)- \E[\nabla_\gamma \psi(W, \gamma, \nu)]\right| \rightarrow^p 0 .
\end{equation*}
The consistency of $\hat{\gamma}$ implies the consistency of $\tilde{\gamma}$, that is, $\tilde{\gamma}\rightarrow^p \gamma_{0}$. 
Also, $\hat{\nu}$ is a consistent estimator, or $\hat{\nu} \to^{p} \nu_{0}$. 
Combining these two properties, we obtain that conditional on the  data set:
\begin{equation*}
\frac{1}{n} \sum_{i=1}^n \nabla_\gamma \psi(W_{i}, \tilde{\gamma}, \hat{\nu})
\rightarrow \E\left[\nabla_\gamma \psi(W,\gamma_0, \nu_0)\right]. 
\end{equation*}
Because the matrix $\E\left[\nabla_\nu \psi(W,\gamma_0,\nu_0)\right]$ is non-singular, the continuity of the inverse operator yields:
\begin{equation*}
\left( 
  \frac{1}{n} \sum_{i=1}^n \nabla_\gamma \psi(W_{i}, \tilde{\gamma}, \nu)
  \right)^{-1}
  \rightarrow^{p} 
  \E\left[\nabla_\gamma\psi(W,\gamma_0, \nu_0)\right]^{-1}. 
\end{equation*}

Next, we shall show the asymptotic normality. 
We apply a second-order Taylor expansion around $\nu_0$:
\begin{align*}
  &
\frac{1}{\sqrt{n}} \sum_{i=1}^n \psi(W_{i}, \gamma_0, \nu) \\
&= \frac{1}{\sqrt{n}} \sum_{i=1}^n \psi(W_{i}, \gamma_0, \nu_0) 
+ \frac{1}{\sqrt{n}} \sum_{i=1}^n \nabla_{\nu} \psi(W_{i}, \gamma_0, \nu_0) 
\left(\hat{\nu}(X_{i}) - \nu_0(X_{i}) \right) \\
& \ \ \ 
+ \frac{1}{2\sqrt{n}} \sum_{i=1}^n 
\big(\hat{\nu}(X_i)-\nu_0(X_i) \big)^{\top}
\nabla_{\nu\nu} \psi(W_{i},\gamma_0, \tilde{\nu}) 
\big(\hat{\nu}(X_i)-\nu_0(X_i) \big),
\end{align*}
where $\tilde{\nu}$ represents some convex combination of $\nu$ and $\nu_0$ by the mean value theorem.
In the left hand side of the above equation, 
observe that the first term represents a normalized sum of $n$ independent and identically distributed random variables, scaled by $\sqrt{n}$. By the Central Limit Theorem, this term converges in distribution to 
$N(0,V)$, where $V:=Var( \psi(W_{i}, \gamma_0, \nu_0) )$.  It remains to show that both the second and third terms converge in probability to zero.

We consider the second term. 
It follows from the Neyman orthogonality condition that 
the conditional mean given the auxiliary information is given by:
\begin{align*}
\E\left[\frac{1}{\sqrt{n}} \sum_{i=1}^n \nabla_{\nu} \psi(W_{i}, \gamma_0, \nu_0) \cdot \left(\hat{\nu}(X_i) - \nu_0(X_i)\right)
\bigg | \hat{\nu}\right] = 0. 
\end{align*}
This implies that the mean of the second term is zero by the law of iterated expectations and also that the expectation of the cross-product of elements in the second term is equal to zero under the iid assumption.
Thus, the conditional covariance of the second term, given the auxiliary dataset, is bounded above by:
\begin{align*}
& \frac{1}{n}\sum_{i=1}^{n}  \E\left[
  \|\nabla_{\nu} \psi(W_{i}, \gamma_0, \nu_0) \cdot \left(\hat{\nu}(X_{i}) - \nu_0(X_{i})\right)
  \|^2~|~\hat{\nu}
  \right] 
  \leq \sigma^2 \E\left[\|\hat{\nu}(X) - \nu_0(X)\|^2\right] .
\end{align*}
The consistency of $\hat{\nu}$ ensures that the latter term converges to zero as the sample size goes to infinity.

Since the mean is zero and the variance converges to zero, Chebyshev's inequality shows that any deviation from the mean must become increasingly rare. This inequality provides a bound on the probability of deviations that is proportional to the variance, which we have shown approaches zero. Therefore, the probability of observing any fixed deviation from zero becomes arbitrarily small as the sample size grows large, establishing convergence in probability to zero.

Now, we consider the third term. 
Since $\nabla_{\nu\nu} \psi(W,\gamma,\nu)$ has a largest eigenvalue uniformly bounded by $\lambda_{\max}$, the third term is bounded by
$(\lambda_{\max}/ 2) n^{-1/2} \sum_{i=1}^n 
\|\hat{\nu}(X_i)-\nu_0(X_i)\|^2 .
$
For a fixed auxiliary data set, 
$n^{-1} \sum_{i=1}^n \|\hat{\nu}(X_i)-\nu_0(X_i)\|^2$ converges to $\E[\|\hat{\nu}(X_i)-\nu_0(X_i)\|^2]$. By the $n^{1/4}$-consistency of the estimator of nuisance functions and its regularity, the third term converges to zero in probability. 
Combining all the results established thus far, we obtain the desired conclusion.
\end{proof}

\end{document}